\documentclass[11pt]{article}
\usepackage{authblk}
\usepackage{hyperref}
\usepackage{graphicx}
\usepackage{subfigure}
\usepackage{algorithm}
\usepackage{algorithmic}

\usepackage{lmodern}
\usepackage[T1]{fontenc}

\usepackage{fullpage}

\usepackage{amssymb}
\usepackage{amsmath}
\usepackage{amsthm}
\usepackage{xspace}

\newtheorem{theorem}{Theorem}
\newtheorem{lemma}[theorem]{Lemma}

\newtheorem{corollary}[theorem]{Corollary}

\newcommand{\R}{{\mathbb{R}}}
\newcommand{\posR}{{\mathbb{R}}_{\ge 0}}
\newcommand{\chr}[1]{\mathbf{1}_{#1}}
\newcommand{\I}{\mathcal{I}}
\newcommand{\sC}{\mathcal{C}}

\DeclareMathOperator*{\Ex}{\mathbb{E}}

\def\opt{\ensuremath{\mathrm{OPT}}\xspace}
\def\set#1{\left\{#1\right\}}
\def\card#1{\left|#1\right|}
\def\Greedy{\textsc{Greedy}\xspace}
\def\GreeDI{\textsc{GreeDi}\xspace}
\def\RandGreeDI{\textsc{RandGreeDi}\xspace}
\def\SamplePrune{\textsc{Sample{\&}Prune}\xspace}
\def\NMRandGreeDI{\textsc{NMRandGreeDi}\xspace}
\def\Alg{\textsc{Alg}\xspace}
\def\etal{\emph{et al.}\xspace}
\def\sep{\;|\;}
\def\vect#1{ {\mathbf{#1}} }
\def\script#1{\mathcal{#1}}

\begin{document}

\title{The Power of Randomization\\
Distributed Submodular Maximization on Massive
Datasets\footnote{The authors are listed alphabetically.}}

\author[1]{Rafael da Ponte Barbosa}
\author[1]{Alina Ene}
\author[2]{Huy L. Nguy\~{\^{e}}n}
\author[1]{Justin Ward\thanks{Work supported by EPSRC grant EP/J021814/1.}}
\affil[1]{Department of Computer Science and DIMAP \authorcr
			University of Warwick \authorcr  {\tt \{rafael, A.Ene, J.D.Ward\}@dcs.warwick.ac.uk}}
\affil[2]{Simons Institute \authorcr University of California,
Berkeley \authorcr {\tt hlnguyen@cs.princeton.edu}}

\maketitle

\begin{abstract}
  A wide variety of problems in machine learning, including exemplar
  clustering, document summarization, and sensor placement,  can be
  cast as constrained submodular maximization problems.
  Unfortunately, the resulting submodular optimization problems are
  often too large to be solved on a single machine. We develop a
  simple distributed algorithm that is embarrassingly parallel and it
  achieves provable, constant factor, worst-case approximation
  guarantees.  In our experiments, we demonstrate its efficiency in
  large problems with different kinds of constraints with objective
  values always close to what is achievable in the centralized
  setting.
\end{abstract}

\section{Introduction}
A set function $f: 2^V \rightarrow \posR$ on a ground set $V$ is
\emph{submodular} if $f(A) + f(B) \geq f(A \cap B) + f(A \cup B)$ for
any two sets $A, B \subseteq V$. Several problems of interest can be
modeled as maximizing a submodular objective function subject to
certain constraints:
	\[ \max f(A) \text{ subject to } A \in \sC, \]
where $\sC \subseteq 2^V$ is the family of feasible solutions.
Indeed, the general meta-problem of optimizing a constrained submodular
function captures a wide variety of problems in machine learning
applications, including exemplar clustering, document summarization,
sensor placement, image segmentation, maximum entropy sampling, and
feature selection problems.  

At the same time, in many of these applications, the amount of data
that is collected is quite large and it is growing at a very fast
pace. For example, the wide deployment of sensors has led to the
collection of large amounts of measurements of the physical world.
Similarly, medical data and human activity data are being captured
and stored at an ever increasing rate and level of detail. This data
is often high-dimensional and complex, and it needs to be stored and
processed in a distributed fashion.

In these settings, it is apparent that the classical algorithmic
approaches are no longer suitable and new algorithmic insights are
needed in order to cope with these challenges. The algorithmic
challenges stem from the following competing demands imposed by huge
datasets: the computations need to process the data that is
distributed across several machines using a minimal amount of
communication and synchronization across the machines, and at the
same time deliver solutions that are competitive with the centralized
solution on the entire dataset.

The main question driving the current work is whether these competing
goals can be reconciled. More precisely, can we deliver very good
approximate solutions with minimal communication overhead? Perhaps
surprisingly, the answer is yes; there is a very simple distributed
greedy algorithm that is embarrassingly parallel and it achieves
provable, constant factor, worst-case approximation guarantees. Our
algorithm can be easily implemented in a parallel model of
computation such as MapReduce \cite{DG08}.

\subsection{Background and Related Work}

In the MapReduce model, there are $m$ independent machines. Each of
the machines has a limited amount of memory available. In our
setting, we assume that the data is much larger than any single
machine's memory and so must be distributed across all of the
machines. At a high level, a MapReduce computation proceeds in
several rounds. In a given round, the data is shuffled among the
machines. After the data is distributed, each of the machines
performs some computation on the data that is available to it. The
output of these computations is either returned as the final result
or becomes the input to the next MapReduce round. We emphasize that
the machines can only communicate and exchange data during the
shuffle phase.

In order to put our contributions in context, we briefly discuss two
distributed greedy algorithms that achieve complementary trade-offs
in terms of approximation guarantees and communication overhead.

Mirzasoleiman \etal \cite{MKSK13} give a distributed algorithm,
called \GreeDI, for maximizing a monotone submodular function subject
to a cardinality constraint. The \GreeDI algorithm partitions the
data arbitrarily on the machines and on each machine it runs the
classical \Greedy algorithm to select a feasible subset of the items
on that machine. The \Greedy solutions on these machines are then
placed on a single machine and the \Greedy algorithm is used once
more to select the final solution. The \GreeDI algorithm is very
simple and embarrassingly parallel, but its worst-case approximation
guarantee\footnote{Mirzasoleiman \etal \cite{MKSK13} give a
family of instances where the approximation achieved is only $1 /
\min\set{k, m}$ if the solution picked on each of the machines is the
optimal solution for the set of items on the machine. These instances
are not hard for the \GreeDI algorithm. We show in Sections
\ref{app:det-greedi-analysis} and \ref{app:det-greedi-tight} that the
\GreeDI algorithm achieves an $1 / \Theta \left(\min\set{\sqrt{k}, m}
\right)$ approximation.} is $1 /\Theta\left(\min\set{\sqrt{k},
m}\right)$, where $m$ is the number of machines and $k$ is the
cardinality constraint. Despite this, Mirzasoleiman \etal show that
the \GreeDI algorithm achieves very good approximations for datasets
with geometric structure.

Kumar \etal \cite{KMVV13} give distributed algorithms for
maximizing a monotone submodular function subject to a cardinality or
more generally, a matroid constraint. Their algorithm combines the
Threshold Greedy algorithm of \cite{GuptaRST10} with a sample and
prune strategy. In each round, the algorithm samples a small subset
of the elements that fit on a single machine and runs the Threshold
Greedy algorithm on the sample in order to obtain a feasible
solution. This solution is then used to prune some of the elements in
the dataset and reduce the size of the ground set. The \SamplePrune
algorithms achieve constant factor approximation guarantees but they
incur a higher communication overhead. For a cardinality constraint,
the number of rounds is a constant but for more general constraints
such as a matroid constraint, the number of rounds is
$\Theta(\log{\Delta})$, where $\Delta$ is the maximum increase in the
objective due to a single element. The maximum increase $\Delta$ can
be much larger than even the number of elements in the entire
dataset, which makes the approach infeasible for massive datasets.

On the negative side, Indyk et al.~\cite{IMMM14} studied coreset
approaches to develop distributed algorithms for finding
representative and yet diverse subsets in large collections. While
succeeding in several measures, they also showed that their approach
provably {\em does not} work for $k$-coverage, which is a special
case of submodular maximization with a cardinality constraint.

\subsection{Our Contribution}
\label{sec:our-contribution}

In this paper, we show that we can achieve both the communication
efficiency of the \GreeDI algorithm and a provable, constant factor,
approximation guarantee. Our algorithm is in fact the \GreeDI
algorithm with a very simple and crucial modification: instead of
partitioning the data arbitrarily on the machines, we \emph{randomly}
partition the dataset. Our analysis may perhaps provide some
theoretical justification for the very good empirical performance of
the \GreeDI algorithm that was established previously in the
extensive experiments of \cite{MKSK13}. It also suggests the approach
can deliver good performance in much wider settings than originally
envisioned.

The \GreeDI algorithm was originally studied in the special case of
monotone submodular maximization under a cardinality constraint.  In
contrast, our analysis holds for any hereditary constraint.
Specifically, we show that our randomized variant of the \GreeDI
algorithm achieves a constant factor approximation for any
hereditary, constrained problem for which the classical (centralized)
\Greedy algorithm achieves a constant factor approximation.  This is
the case not only for cardinality constraints, but also for matroid
constraints, knapsack constraints, and $p$-system constraints
\cite{Jenkyns1976}, which generalize the intersection of $p$ matroid
constraints.  Table \ref{tab:monotone} gives the approximation ratio
$\alpha$ obtained by the greedy algorithm on a variety of problems,
and the corresponding constant factor obtained by our randomized
\GreeDI algorithm.

\begin{table}
\let\oldarraystretch=\arraystretch
\renewcommand{\arraystretch}{1.13}
\centering
\begin{tabular}{|c|c|c|c|}
\hline
Constraint & $\alpha$ & monotone approx. $\left(\frac{\alpha}{2}\right)$& non-monotone approx. $\left(\frac{\alpha}{4 + 2\alpha}\right)$\\ \hline
cardinality & $1 - \frac{1}{e} \approx 0.632$ & $\approx 0.316 $ & $\approx $ 0.12\\ \hline
matroid & $\frac{1}{2}$ & $\frac{1}{4}$ & $\frac{1}{10}$ \\ \hline
knapsack & $\approx 0.35$ & $\approx 0.17$ & $\approx 0.074$ \\ \hline
$p$-system & $\frac{1}{p+1}$& $\frac{1}{2(p+1)}$ & $\frac{1}{2 + 4(p+1)}$ \\ \hline
\end{tabular}
\caption{New approximation results for randomized \GreeDI for 
constrained monotone and non-monotone submodular  maximization\protect\footnotemark}
\label{tab:monotone}
\let\arraystretch=\oldarraystretch
\end{table}

Additionally, we show that if the greedy algorithm satisfies a
slightly stronger technical condition, then our approach gives a
constant factor approximation for constrained \emph{non-monotone}
submodular maximization.  This is indeed the case for all of the
aforementioned specific classes of problems.  The resulting
approximation ratios for non-monotone maximization problems are given
in the last column of Table \ref{tab:monotone}.

\footnotetext{The best-known values of $\alpha$ are taken from
\cite{NWF78} (cardinality), \cite{FNW78-II} (matroid and $p$-system),
and \cite{Wolsey1982} (knapsack).  In the case of a knapsack
constraint, Wolsey in fact employs a slightly modified variant of the
greedy algorithm.  We note that the modified algorithm still
satisfies all technical conditions required for our analysis (in
particular, those for Lemma \ref{lem:rejected-elements}).}

\subsection{Preliminaries}

{\bf MapReduce Model.}
In a MapReduce computation, the data is represented as $\left<
\mathrm{key}, \mathrm{value} \right>$ pairs and it is distributed
across $m$ machines. The computation proceeds in rounds. In a given,
the data is processed in parallel on each of the machines by
\emph{map tasks} that output $\left< \mathrm{key}, \mathrm{value}
\right>$ pairs. These pairs are then shuffled by \emph{reduce tasks};
each reduce task processes all the $\left< \mathrm{key},
\mathrm{value} \right>$ pairs with a given key. The output of the
reduce tasks either becomes the final output of the MapReduce
computation or it serves as the input of the next MapReduce round.

{\bf Submodularity.}
As noted in the introduction, a set function $f : 2^V \to \posR$ is
\emph{submodular} if, for all sets $A,B \subseteq V$, 
\[
	f(A) + f(B) \ge f(A \cup B) + f(A \cap B).
\]
A useful alternative characterization of submodularity can be
formulated in terms of diminishing marginal gains.  Specifically, $f$
is submodular if and only if:
\[
	f(A \cup \set{e}) - f(A) \ge f(B \cup \set{e}) - f(B)
\]
for all $A \subseteq B \subseteq V$ and $e \notin B$.

The \emph{Lov\'asz extension} $f^- : [0,1]^V \to \posR$ of a
submodular function $f$ is given by:
\[
f^-(\vect{x}) = \Ex_{\theta \in \script{U}(0,1)}[f(\{i : x_i \ge
\theta\})].
\]
For any submodular function $f$, the Lov\'asz extension $f^-$ satisfies the following properties: (1) $f^-(\chr{S}) = f(S)$ for all $S \subseteq V$, (2) $f^-$ is convex, and (3) $f^-(c \cdot \vect{x}) \ge c \cdot f^-(\vect{x})$ for any $c \in [0,1]$.  These three properties immediately give the following simple lemma:
\begin{lemma}
Let $S$ be a random set, and suppose that $\Ex[\chr{S}] = c \cdot \vect{p}$ (for $c \in [0,1]$).  Then, $\Ex[f(S)] \ge c \cdot f^-(\vect{p})$.
\label{lem:lovasz}
\end{lemma}
\begin{proof} We have:
\begin{equation*}
\Ex[f(S)] = \Ex[f^-(\chr{S})] \ge f^-(\Ex[\chr{S}]) 
= f^-(c \cdot \vect{p}) \ge c \cdot f^-(\vect{p}),
\end{equation*}
where the first equality follows from property (1), the first inequality from property (2), and the final inequality from property (3).
\end{proof}

{\bf Hereditary Constraints.}
Our results hold quite generally for any problem which can be formulated in terms of a hereditary constraint.  Formally, we consider the problem 
\begin{equation}
\label{eq:problem}
\max \{f(S) : S \subseteq V, S \in \I\},
\end{equation}
where $f : 2^V \to \posR$ is a submodular function and $\I \subseteq
2^V$ is a family of feasible subsets of $V$.  We require that $\I$ be
\emph{hereditary} in the sense that if some set is in $\I$, then so
are all of its subsets.  Examples of common hereditary families
include cardinality constraints ($\I = \{ A \subseteq V : |A| \le k
\}$), matroid constraints ($\I$ corresponds to the collection
independent sets of the matroid), knapsack constraints ($\I = \{ A
\subseteq V : \sum_{i \in A}w_i \le b\}$), as well as arbitrary
combinations of such constraints.  Given some constraint $\I
\subseteq 2^V$, we shall also consider restricted instances in which
we are presented only with a subset $V' \subseteq V$, and must find a
set $S \subseteq V'$ with $S \in \I$ that maximizes $f$.  We say that
an algorithm is an $\alpha$-approximation for maximizing a submodular
function subject to a hereditary constraint $\I$ if, for any
submodular function $f : 2^V \to \posR$ and any subset $V' \subseteq
V$ the algorithm produces a solution $S \subseteq V'$ with $S \in
\I$, satisfying $f(S) \ge \alpha \cdot f(\opt)$, where $\opt \in \I$
is any feasible subset of $V'$.

\section{The Standard Greedy Algorithm}
\begin{algorithm}[t]
\caption{The standard greedy algorithm $\Greedy$}
\label{alg:greedy}
\begin{algorithmic}
  \STATE $S \gets \emptyset$
\LOOP
\STATE Let $C = \{e \in V \setminus S : S \cup \set{e} \in \I \}$
\STATE Let $e = \arg\max_{e \in C}\{f(S \cup \set{e}) - f(S)\}$
\IF{$C = \emptyset$ or $f(S \cup \set{e}) - f(S) < 0$}
\STATE\textbf{return} $S$
\ENDIF
\ENDLOOP
\end{algorithmic}
\end{algorithm}
\begin{algorithm}[t]
\caption{The distributed algorithm $\RandGreeDI$}
\label{alg:randgreedi}
\begin{algorithmic}
\FOR{$e \in V$}
\STATE Assign $e$ to a machine $i$ chosen uniformly at random
\ENDFOR
\STATE Let $V_i$ be the elements assigned to machine $i$
\STATE Run $\Greedy(V_i)$ on each machine $i$ to obtain $S_i$
\STATE Place $S = \bigcup_i S_i$ on machine $1$
\STATE Run $\Alg(S)$ on machine $1$ to obtain $T$
\STATE Let $S' = \arg\max_i\{f(S_i)\}$
\STATE \textbf{return} $\arg\max \{f(T), f(S')\}$
\end{algorithmic}
\end{algorithm}

Before describing our general algorithm, let us recall the standard
greedy algorithm, \Greedy, shown in Algorithm \ref{alg:greedy}.  The
algorithm takes as input $\langle V, \I, f \rangle$, where $V$ is a
set of elements, $\I \subseteq 2^V$ is a hereditary constraint,
represented as a membership oracle for $\I$, and $f: 2^V \to \posR$
is a non-negative submodular function, represented as a value oracle.
Given $\langle V, \I, f\rangle$,  \Greedy iteratively constructs a
solution $S \in \I$ by choosing at each step the element maximizing
the marginal increase of $f$.  For some $A \subseteq V$, we let
$\Greedy(A)$ denote the set $S \in \I$ produced by the greedy
algorithm that considers only elements from $A$.

The greedy algorithm satisfies the following property:

\begin{lemma} \label{lem:rejected-elements}
	Let $A \subseteq V$ and $B \subseteq V$ be two disjoint subsets
	of $V$. Suppose that, for each element $e \in B$, we have
	$\Greedy(A \cup \set{e}) = \Greedy(A)$. Then $\Greedy(A \cup B) =
	\Greedy(A)$.	
\end{lemma}
\begin{proof}
  Suppose for contradiction that $\Greedy(A \cup B) \neq
	\Greedy(A)$. We first note that, if $\Greedy(A \cup B) \subseteq
	A$, then $\Greedy(A \cup B) = \Greedy(A)$; this follows from the
	fact that each iteration of the Greedy algorithm chooses the
	element with the highest marginal value whose addition to the
	current solution maintains feasibility for $\I$. Therefore, if $\Greedy(A
	\cup B) \neq \Greedy(A)$, the former solution contains an element
	of $B$. Let $e$ be the first element of $B$ that is selected by
	Greedy on the input $A \cup B$. Then Greedy will also select $e$
	on the input $A \cup \set{e}$, which contradicts the fact that
	$\Greedy(A \cup \set{e}) = \Greedy(A)$.
\end{proof}

\section{A Randomized, Distributed Greedy Algorithm for Monotone
Submodular Maximization}
\label{sec:rand-distr-greedy}

\textbf{Algorithm.}
We now describe our general, randomized distributed algorithm,
\RandGreeDI, shown in Algorithm \ref{alg:randgreedi}.  Suppose we
have $m$ machines.  Our algorithm runs in two rounds.  In the first
round, we \emph{randomly} distribute the elements of the ground set
$V$ to the machines, assigning each element to a machine chosen
independently and uniformly at random. On each machine $i$, we
execute $\Greedy(V_i)$ to select a feasible subset $S_i$ of the
elements on that machine.  In the second round, we place all of these
selected subsets on a single machine, and run some algorithm \Alg on
this machine in order to select a final solution $T$.  We return
whichever is better: the final solution $T$ or the best solution
amongst all the $S_i$ from the first phase.

\textbf{Analysis.}
We devote the rest of this section to the analysis of the \RandGreeDI
algorithm.  Fix $\langle V, \I, f\rangle$, where $\I \subseteq 2^V$
is a hereditary constraint, and $f : 2^V\to \posR$ is any
non-negative, monotone submodular function.  Suppose that \Greedy is
an $\alpha$-approximation and \Alg is a $\beta$-approximation for the
associated constrained monotone submodular maximization problem of
the form \eqref{eq:problem}.   Let $n = |V|$ and suppose that
$\opt = \arg\max_{A \in \I}f(A)$ is a feasible set maximizing $f$.


Let $\mathcal{V}(1/m)$ denote the distribution over random subsets of $V$ where each element is included
independently with probability $1/m$.  Let $\vect{p} \in [0, 1]^n$ be
the following vector. For each element $e \in V$, we have
\begin{equation*}
	p_e = \begin{cases}
    \underset{A \sim \mathcal{V}(1/m)}{\Pr}[e \in \Greedy(A \cup \set{e})] &
    \text{if $e \in \opt$}\\
		0 & \text{otherwise}
	\end{cases}
\end{equation*}
Our main theorem follows from the next two lemmas, which characterize
the quality of the best solution from the first round and that of the
solution from the second round, respectively. Recall that $f^-$ is
the Lov\'asz extension of $f$.

\begin{lemma} \label{lem:single-machine-greedy}
	For each machine $i$, $\Ex[f(S_i)] \geq \alpha \cdot
	f^-\left(\vect{1}_{\opt} - \vect{p} \right).$
\end{lemma}
\begin{proof}
	Consider machine $i$. Let $V_i$ denote the set of elements assigned
	to machine $i$ in the first round. Let $O_i = \set{e \in \opt
	\colon e \notin \Greedy(V_i \cup \set{e})}$. We make the following
	key observations.

	We apply Lemma~\ref{lem:rejected-elements} with $A = V_i$ and $B =
	O_i \setminus V_i$ to obtain that $\Greedy(V_i) = \Greedy(V_i \cup
	O_i) = S_i$. Since $\opt \in \I$ and $\I$ is hereditary, we must
	have $O_i \in \I$ as well. Since $\Greedy$ is an
	$\alpha$-approximation, it follows that
	\[
		f(S_i) \ge \alpha \cdot f(O_i).
	\]
	Since the distribution of $V_i$ is the same as $\mathcal{V}(1/m)$, for each element
	$e \in \opt$, we have
	\begin{align*}
		\Pr[e \in O_i] &= 1 - \Pr[e \notin O_i] = 1 - p_e\\
		\Ex[\chr{O_i}] &= \chr{\opt} - \vect{p}.
	\end{align*}
	By combining these observations with Lemma~\ref{lem:lovasz}, we
	obtain
	\[
		\Ex[f(S_i)] \ge \alpha \cdot \Ex[f(O_i)] \ge \alpha \cdot
		f^-\left(\vect{1}_{\opt} - \vect{p}\right).
	\]
\end{proof}

\begin{lemma} \label{lem:union-greedy}
	$\Ex[f(\Alg(S))] \geq \beta \cdot f^-(\vect{p}).$
\end{lemma}
\begin{proof}
	Recall that $S = \bigcup_i \Greedy(V_i)$. Since $\opt \in \I$ and
	$\I$ is hereditary, $S \cap \opt \in \I$. Since $\Alg$ is a
	$\beta$-approximation, we have
	\begin{equation}
		f(\Alg(S)) \geq \beta \cdotp f(S \cap \opt). \label{eq:mg1}
	\end{equation}
	Consider an element $e \in \opt$. For each machine $i$, we have
	\begin{align*}
		\Pr[e \in S \sep \text{$e$ is assigned to machine $i$}] &= \Pr[e
		\in \Greedy(V_i) \sep e \in V_i]\\
		&= \Pr_{A \sim \mathcal{V}(1/m)}[e \in \Greedy(A) \sep e \in A]\\
		&= \Pr_{B \sim \mathcal{V}(1/m)}[e \in \Greedy(B \cup \set{e} )]\\
		&= p_e.
	\end{align*}
  The first equality follows from the fact that $e$ is included in
  $S$ if and only if it is included in $\Greedy(V_i)$. The second
  equality follows from the fact that the distribution of $V_i$ is
  identical to $\mathcal{V}(1/m)$. The third equality follows from
  the fact that the distribution of $A \sim \mathcal{V}(1/m)$
  conditioned on $e\in A$ is identical to the distribution of $B \cup
  \{e\}$ where $B\sim \mathcal{V}(1/m)$. Therefore
	\begin{align}
		\Pr[e \in S \cap \opt] &= p_e \notag\\
		\Ex[\chr{S \cap \opt}] &= \vect{p}. \label{eq:mg2}
	\end{align}
	By combining (\ref{eq:mg1}), (\ref{eq:mg2}), and
	Lemma~\ref{lem:lovasz}, we obtain
	\[
		\Ex[f(\Alg(S))] \geq \beta \cdotp \Ex[f(S \cap \opt)] \geq \beta
		\cdot f^-(\vect{p}).
	\]
\end{proof}

Combining Lemma~\ref{lem:union-greedy} and
Lemma~\ref{lem:single-machine-greedy} gives us our main theorem.

\begin{theorem}
\label{thm:main}
	Suppose that \Greedy is an $\alpha$-approximation algorithm and
	\Alg is a $\beta$-approximation algorithm for maximizing a
	monotone submodular function subject to a hereditary constraint
	$\I$.  Then \RandGreeDI is (in expectation) an
	$\frac{\alpha\beta}{\alpha + \beta}$-approximation algorithm for
	the same problem.
\end{theorem}
\begin{proof}
	Let $S_i = \Greedy(V_i)$, $S = \bigcup_iS_i$ be the set of
	elements on the last machine, and $T = \Alg(S)$ be the solution
	produced on the last machine.  Then, the output $D$ produced by
	\RandGreeDI satisfies $f(D) \ge \max_i(f(S_i))$ and $f(D) \ge
	f(T)$. Thus, from Lemmas \ref{lem:single-machine-greedy} and
	\ref{lem:union-greedy} we have:
	\begin{align}
		\Ex[f(D)] &\ge \alpha \cdot f^-(\chr{\opt} - \vect{p})
		\label{eq:D1}\\
		\Ex[f(D)] &\ge \beta \cdot f^-(\vect{p}).
		\label{eq:D2}
	\end{align}
	By combining \eqref{eq:D1} and \eqref{eq:D2}, we obtain
	\begin{align*}
		\left(\beta + \alpha\right)\Ex[f(D)] &\ge \alpha\beta
		\big(f^-(\vect{p}) + f^-(\chr{\opt} - \vect{p}) \big) \\
		& \ge \alpha\beta \cdot f^-(\chr{\opt}) \\
		& = \alpha\beta \cdot f(\opt).
	\end{align*}
	In the second inequality, we have used the fact that $f^-$ is
	convex and $f^-(c \cdotp \vect{x}) \geq c f^-(\vect{x})$ for any
	constant $c \in [0, 1]$.
\end{proof}

If we use the standard greedy algorithm for $\Alg$, we obtain the
following simplified corollary of Theorem \ref{thm:main}.

\begin{corollary}
\label{cor:randomized-greedi-approx}
	Suppose that \Greedy is an $\alpha$-approximation algorithm for
	maximizing a monotone submodular function, and use \Greedy as
	the algorithm \Alg in \RandGreeDI.  Then, the resulting algorithm
	is (in expectation) an $\frac{\alpha}{2}$-approximation
	algorithm for the same problem.
\end{corollary}

\section{Non-Monotone Submodular Functions}

We consider the problem of maximizing a \emph{non-monotone}
submodular function subject to a hereditary constraint.  Our approach
is a slight modification of the randomized, distributed greedy
algorithm described in Section \ref{sec:rand-distr-greedy}, and it
builds on the work of \cite{GuptaRST10}.  Again, we show how to
combine the standard \Greedy algorithm, together with any algorithm
\Alg for the non-monotone case in order to obtain a randomized,
distributed algorithm for the non-monotone submodular maximization.

\textbf{Algorithm.} Our modified  algorithm, \NMRandGreeDI, works as
follows.  As in the monotone case, in the first round we distribute
the elements of $V$ uniformly at random amongst the $m$ machines.
Then, we run the standard greedy algorithm \emph{twice} to obtain two
disjoint solutions $S_i^1$ and $S_i^2$ on each machine.
Specifically, each machine first runs $\Greedy$ on $V_i$ to obtain a
solution $S_i^1$, then runs \Greedy on $V_i \setminus S_i^1$ to
obtain a disjoint solution $S_i^2$.  In the second round, both of
these solutions are sent to a single machine, which runs \Alg on $S =
\bigcup_i(S_i^1 \cup S_i^2)$ to produce a solution $T$.  The best
solution amongst $T$ and all of the solutions $S_i^1$ and $S_i^2$ is
then returned.

\textbf{Analysis.}
We devote the rest of this section to the analysis of the algorithm.
In the following, we assume that we are working with an instance
$\langle V, \I, f \rangle$ of non-negative, non-monotone submodular
maximization for which the \Greedy algorithm has the following
property:
\begin{equation}
	\text{For all $S \in \I$:} 
	\qquad f(\Greedy(V)) \ge \alpha \cdot f(\Greedy(V) \cup S)
	\tag{\ensuremath{\mathrm{GP}}}
	\label{eq:greedy-strong-property}
\end{equation}
The standard analysis of the \Greedy algorithm shows that
(\ref{eq:greedy-strong-property}) is satisfied with constant $\alpha$
for hereditary constraints such as matroids, knapsacks, and
$p$-systems (see Table~\ref{tab:monotone}).

The analysis is similar to the approach from the previous section. We
define $\script{V}(1/m)$ as before. We modify the definition of the
vector $\vect{p}$ as follows. For each element $e \in V$, we have
\begin{equation*}
	p_e =
  \begin{cases}
    \underset{A\sim \mathcal{V}(1/m)}{\Pr}\Big[e \in \Greedy(A \cup
    \{e\}) \textbf{ or } & \\
    \qquad\qquad\; e \in \Greedy ((A \cup \{e\})\setminus \Greedy(A\cup
    \{e\})) \Big] & \text{if $e \in \opt$}\\
		0 & \text{otherwise}
	\end{cases}
\end{equation*}

We now derive analogues of Lemmas~\ref{lem:single-machine-greedy} and
\ref{lem:union-greedy}.
\begin{lemma} \label{lem:non-monotone-single-machine}
	Suppose that \Greedy satisfies \eqref{eq:greedy-strong-property}.
	For each machine $i$,
	\[
		\Ex\left[f(S^1_i) + f(S^2_i) \right] \ge
		\alpha \cdotp f^-(\vect{1}_{\opt} - \vect{p}),
	\]
	and therefore
	\[
		\Ex\left[\max\set{f(S^1_i), f(S^2_i)}\right] \ge
		\frac{\alpha}{2} \cdotp f^-(\vect{1}_{\opt} - \vect{p}).
	\]
\end{lemma}
\begin{proof}
	Consider machine $i$ and let $V_i$ be the set of elements assigned
	to machine $i$ in the first round. Let
	\begin{align*}
		O_i = \{e \in \opt \colon e &\notin \Greedy(V_i \cup \set{e})
		\textbf{ and } \\
		e &\notin \Greedy ((V_i \cup \set{e} ) \setminus
		\Greedy(V_i \cup \set{e})) \}
	\end{align*}
	Note that, since $\opt \in \I$ and $\I$ is hereditary, we have $O_i
	\in \I$.

	It follows from Lemma~\ref{lem:rejected-elements} that
	\begin{align}
		S^1_i &= \Greedy(V_i) = \Greedy(V_i \cup O_i)
		\label{eq:rejected1},\\
		S^2_i &= \Greedy(V_i \setminus S^1_i) = \Greedy((V_i \setminus
		S^1_i) \cup O_i). \label{eq:rejected2}
	\end{align}
	By combining the equations above with the greedy property
	(\ref{eq:greedy-strong-property}), we obtain
	\begin{align}
		f(S^1_i) &\overset{(\ref{eq:rejected1})}{=} f(\Greedy(V_i \cup
		O_i)) \notag\\
		& \overset{(\ref{eq:greedy-strong-property})}{\geq} \alpha \cdotp
		f(\Greedy(V_i \cup O_i) \cup O_i) \notag\\
		&\overset{(\ref{eq:rejected1})}{=} \alpha \cdotp f(S^1_i \cup
		O_i) \label{eq:g1},\\
		f(S^2_i) &\overset{(\ref{eq:rejected2})}{=} f(\Greedy((V_i \setminus
		S^1_i) \cup O_i)) \notag\\
		&\overset{(\ref{eq:greedy-strong-property})}{\geq} \alpha \cdotp
		f(\Greedy((V_i \setminus S^1_i) \cup O_i) \cup O_i) \notag\\
		&\overset{(\ref{eq:rejected2})}{=} \alpha \cdotp f(S^2_i \cup
		O_i). \label{eq:g2}
	\end{align}
	Now we observe that
	\begin{align}
		f(S^1_i \cup O_i) + f(S^2_i \cup O_i) &\geq f( (S^1_i \cup O_i)
		\cap (S^2_i \cup O_i)) + f(S^1_i \cup S^2_i \cup O_i) &
		\mbox{($f$ is submodular)} \notag\\
		&= f(O_i) + f(S^1_i \cup S^2_i \cup O_i) &
		\mbox{($S^1_i \cap S^2_i = \emptyset$)} \notag\\
		&\geq f(O_i). & \mbox{($f$ is non-negative)} \label{eq:g3}
	\end{align}
	By combining (\ref{eq:g1}), (\ref{eq:g2}), and (\ref{eq:g3}), we
	obtain
	\begin{equation}
		f(S^1_i) + f(S^2_i) \geq \alpha \cdotp f(O_i). \label{eq:g4}
	\end{equation}
	Since the distribution of $V_i$ is the same as $\mathcal{V}(1/m)$, for each element
	$e \in \opt$, we have
	\begin{align}
		\Pr[e \in O_i] &= 1 - \Pr[e \notin O_i] = 1 - p_e, \notag\\
		\Ex[\chr{O_i}] &= \chr{\opt} - \vect{p}.
		\label{eq:g5}
	\end{align}
	By combining (\ref{eq:g4}), (\ref{eq:g5}), and
	Lemma~\ref{lem:lovasz}, we obtain
	\begin{align*}
		\Ex[f(S^1_i) + f(S^2_i)] &\geq \alpha \cdotp \Ex[f(O_i)]
		&\mbox{(By (\ref{eq:g4}))}\\
		&\geq \alpha \cdotp f^-(\chr{\opt} - \vect{p}). &\mbox{(By
		(\ref{eq:g5}) and Lemma~\ref{lem:lovasz})}
	\end{align*}
\end{proof}
\begin{lemma} \label{lem:non-monotone-union-greedy}
	$\Ex[f(\Alg(S))] \geq \beta \cdot f^-(\vect{p}).$
\end{lemma}
\begin{proof}
	Recall that $S^1_i = \Greedy(V_i)$, $S^2_i = \Greedy(V_i
	\setminus S^1_i)$, and $S = \bigcup_i (S^1_i \cup S^2_i)$. Since
	$\opt \in \I$ and $\I$ is hereditary, $S \cap \opt \in \I$. Since
	$\Alg$ is a $\beta$-approximation, we have
	\begin{equation}
		f(\Alg(S)) \geq \beta \cdotp f(S \cap \opt). \label{eq:nmg1}
	\end{equation}
	Consider an element $e \in \opt$. For each machine $i$, we have
	\begin{align*}
		& \Pr [e \in S \sep \text{$e$ is assigned to machine $i$}] \\
		& \quad = \Pr[e \in \Greedy(V_i) \textbf{ or } e \in \Greedy(V_i \setminus
		\Greedy(V_i)) \sep e \in V_i]\\
    & \quad =\Pr_{A\sim \mathcal{V}(1/m)}[e \in \Greedy(A) \textbf{
    or } e \in \Greedy(A \setminus \Greedy(A)) \sep e \in A]\\
    & \quad = \Pr_{B \sim \mathcal{V}(1/m)}[e \in
    \Greedy(B\cup\set{e}) \textbf{ or } e \in \Greedy((B\cup\set{e})
    \setminus \Greedy(B\cup\set{e})) ]\\
		&\quad = p_e.
	\end{align*}
  The first equality above follows from the fact that $e$ is included
  in $S$ iff $e$ is included in either $S_i^1$ or $S_i^2$. The second
  equality follows from the fact that the distribution of $V_i$ is
  the same as $\mathcal{V}(1/m)$.  The third equality follows from
  the fact that the distribution of $A\sim\mathcal{V}(1/m)$
  conditioned on $e\in A$ is identical to the distribution of $B\cup
  \{e\}$ where $B\sim\mathcal{V}(1/m)$. Therefore
	\begin{align}
		\Pr[e \in S \cap \opt] &= p_e, \notag\\
		\Ex[\chr{S \cap \opt}] &= \vect{p}. \label{eq:nmg2}
	\end{align}
	By combining (\ref{eq:nmg1}), (\ref{eq:nmg2}), and
	Lemma~\ref{lem:lovasz}, we obtain
	\[
		\Ex[f(\Alg(S))] \geq \beta \cdotp \Ex[f(S \cap \opt)] \geq \beta
		\cdot f^-(\vect{p}).
	\]

\end{proof}
We can now combine Lemmas \ref{lem:non-monotone-union-greedy}  and
\ref{lem:non-monotone-single-machine} to obtain our main result for
non-monotone submodular maximization.

\begin{theorem}
\label{thm:main-non-monotone}
	Consider the problem of maximizing a submodular function under
	some hereditary constraint $\I$, and suppose that \Greedy
	satisfies \eqref{eq:greedy-strong-property} and \Alg is a
	$\beta$-approximation algorithm for this problem.  Then
	\NMRandGreeDI is (in expectation) an $\frac{\alpha\beta}{\alpha
	+ 2\beta}$-approximation algorithm for the same problem.
\end{theorem}
\begin{proof}
	Let $S^1_i = \Greedy(V_i)$, $S^2_i = \Greedy(V_i \setminus
	S^1_i)$, and $S = \bigcup_i(S^1_i \cup S^2_i)$ be the set of
	elements on the last machine, and $T = \Alg(S)$ be the solution
	produced on the last machine.  Then, the output $D$ produced by
	\RandGreeDI satisfies $f(D) \ge \max_i\max\{f(S^1_i), f(S^2_i)\}$
	and $f(D) \ge f(T)$.  Thus, from Lemmas
	\ref{lem:non-monotone-single-machine} and
	\ref{lem:non-monotone-union-greedy} we have:
	\begin{align}
		\Ex[f(D)] &\ge \frac{\alpha}{2} \cdot f^-(\chr{\opt} -
		\vect{p}), \label{eq:D1-nm}\\
		\Ex[f(D)] &\ge \beta \cdot
		f^-(\vect{p}).\label{eq:D2-nm}
	\end{align}
	By combining \eqref{eq:D1-nm} and \eqref{eq:D2-nm}, we obtain
	\begin{align*}
		\left(2\beta + \alpha\right)\Ex[f(D)] &\ge
		\alpha\beta[f^-(\vect{p}) + f^-(\chr{\opt} - \vect{p})]\\
		&\ge \alpha\beta \cdot f^-(\chr{\opt})\\
		&= \alpha \beta \cdot f(\opt).
	\end{align*}
	In the second inequality, we have used the fact that $f^-$ is
	convex and $f^-(c \cdotp \vect{x}) \geq c f^-(\vect{x})$ for any
	constant $c \in [0, 1]$.
\end{proof}
We remark that one can use the following approach on the last machine
\cite{GuptaRST10}. As in the first round, we run \Greedy twice to
obtain two solutions $T_1 = \Greedy(S)$ and $T_2 = \Greedy(S
\setminus T_1)$. Additionally, we select a subset $T_3 \subseteq T_1$
using an \emph{unconstrained} submodular maximization algorithm on
$T_1$, such as the Double Greedy algorithm of \cite{BuchbinderFNS12},
which is a $\frac{1}{2}$-approximation.  The final solution $T$ is
the best solution among $T_1, T_2, T_3$. If \Greedy satisfies
property \ref{eq:greedy-strong-property}, then it follows from the
analysis of \cite{GuptaRST10} that the resulting solution $T$
satisfies $f(T) \ge \frac{\alpha}{2(1 + \alpha)}\cdot f(\opt)$.  This
gives us the following corollary of Theorem
\ref{thm:main-non-monotone}:
\begin{corollary}
\label{cor:non-monotone-randomized-greedi-approx}
  Consider the problem of maximizing a submodular function subject to
  some hereditary constraint $\I$ and suppose that \Greedy satisfies
  \eqref{eq:greedy-strong-property} for this problem. Let \Alg be the
  algorithm described above that uses \Greedy twice and Double
  Greedy.  Then \NMRandGreeDI achieves (in expectation) an
  $\frac{\alpha}{4 + 2\alpha}$-approximation for the same problem.
\end{corollary}
\begin{proof}
By \eqref{eq:greedy-strong-property} and the approximation guarantee of the Double Greedy algorithm, we have:
\begin{align}
f(T) \ge f(T_1) &\ge \alpha \cdot f(T_1 \cup \opt) \label{eq:corrT1} \\
f(T) \ge f(T_2) &\ge \alpha \cdot f(T_2 \cup (\opt \setminus T_1)) \label{eq:corrT2}\\
f(T) \ge f(T_3) &\ge \frac{1}{2}f(T_1 \cap \opt). \label{eq:corrT3}
\end{align}
Additionally, from \cite[Lemma 2]{GuptaRST10}, we have:
\begin{equation*}
f(T_1 \cup \opt) + f(T_2 \cup (\opt \setminus T_1)) + f(T_1 \cap \opt) \ge f(\opt)
\end{equation*}
By combining the inequalities above, we obtain:
\begin{equation*}
(1 + \alpha)f(T) \ge 
\frac{\alpha}{2}\left( f(T_1 \cup \opt) + 
 f(T_2 \cup (\opt \setminus T_1)) + 
 f(T_1 \cap \opt) \right)
\ge \frac{\alpha}{2}f(\opt)
\end{equation*}
and hence $f(T) \ge \frac{\alpha}{2(1 + \alpha)}\cdot f(\opt)$ as claimed.  Setting $\beta = \frac{\alpha}{2(\alpha + 1)}$ in Theorem \ref{thm:main-non-monotone}, we obtain an approximation ratio of $\frac{\alpha}{4 + 2\alpha}$.
\end{proof}

\section{Experiments}
\label{sec:experiments}
\begin{figure*}[p]
\subfigure[][Kosarak dataset]{
\begin{minipage}[t]{0.31\textwidth}
  \centering
  \includegraphics[scale=0.29]{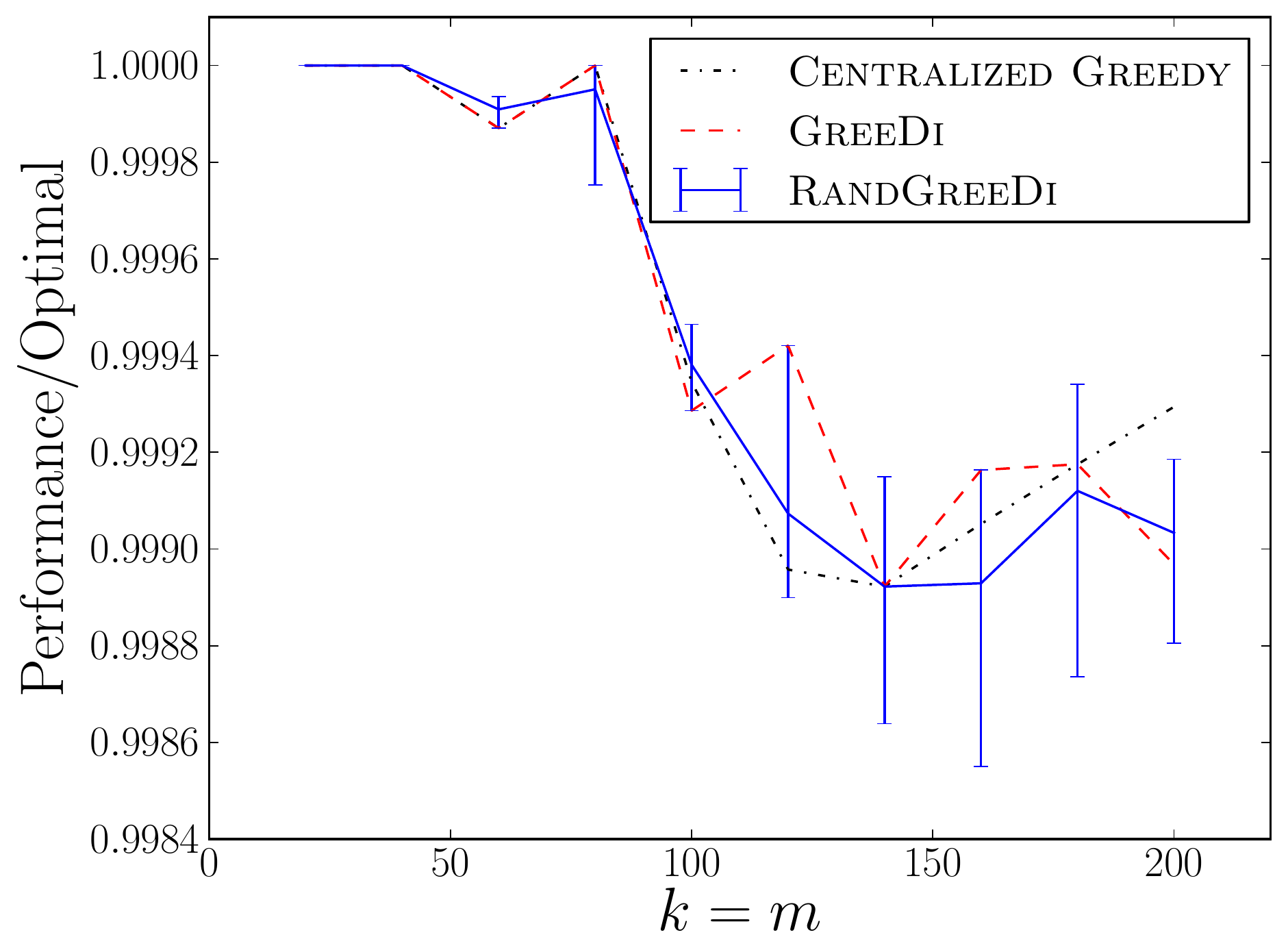}
\end{minipage}
\label{fig:kosarak-rand-det-central}
}
\subfigure[][accidents dataset]{
\begin{minipage}[t]{0.31\textwidth}
  \centering
  \includegraphics[scale=0.29]{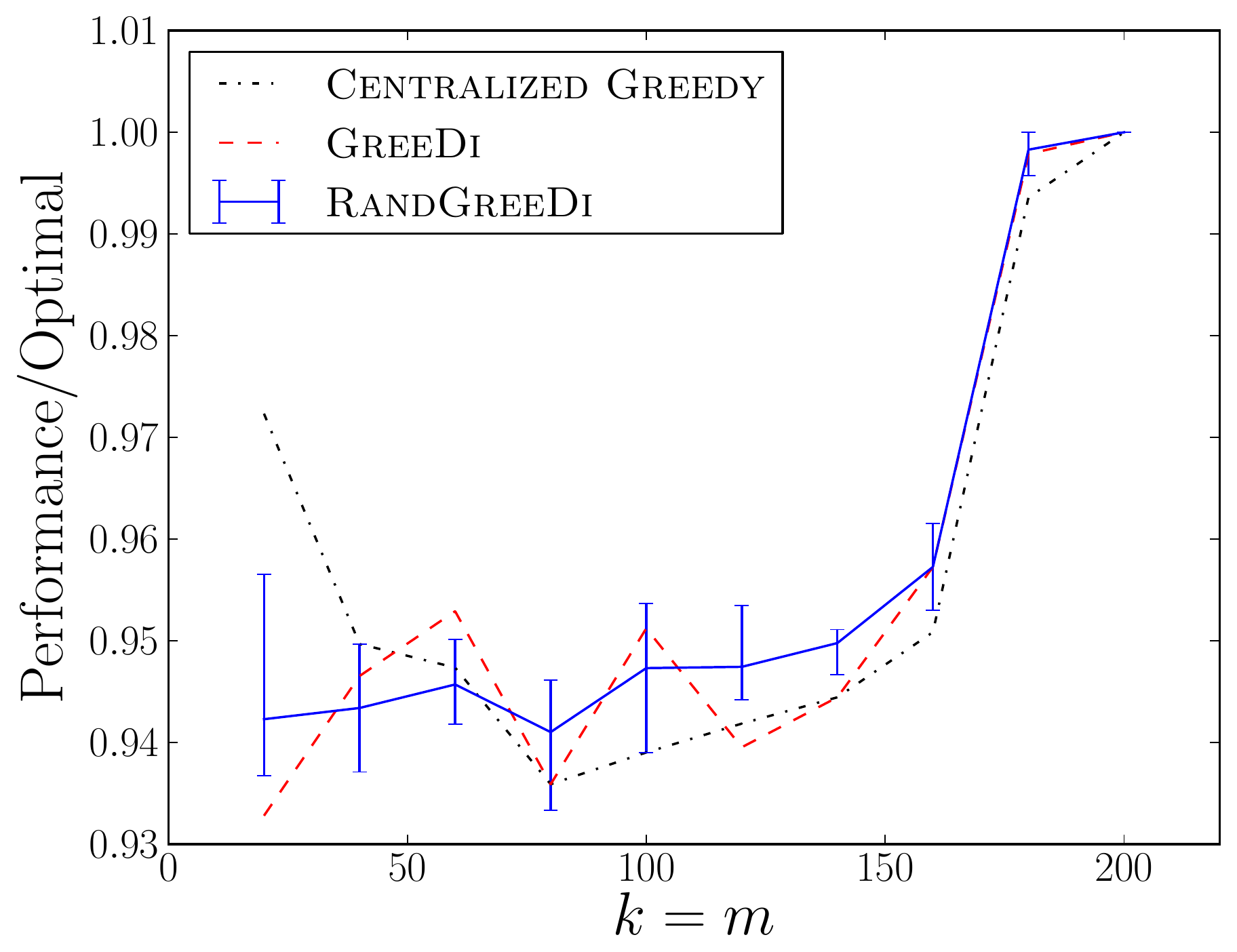}
\end{minipage}
\label{fig:acc-rand-det-central}
}
\subfigure[][10K tiny images]{
\begin{minipage}[t]{0.31\textwidth}
  \centering
  \includegraphics[scale=0.29]{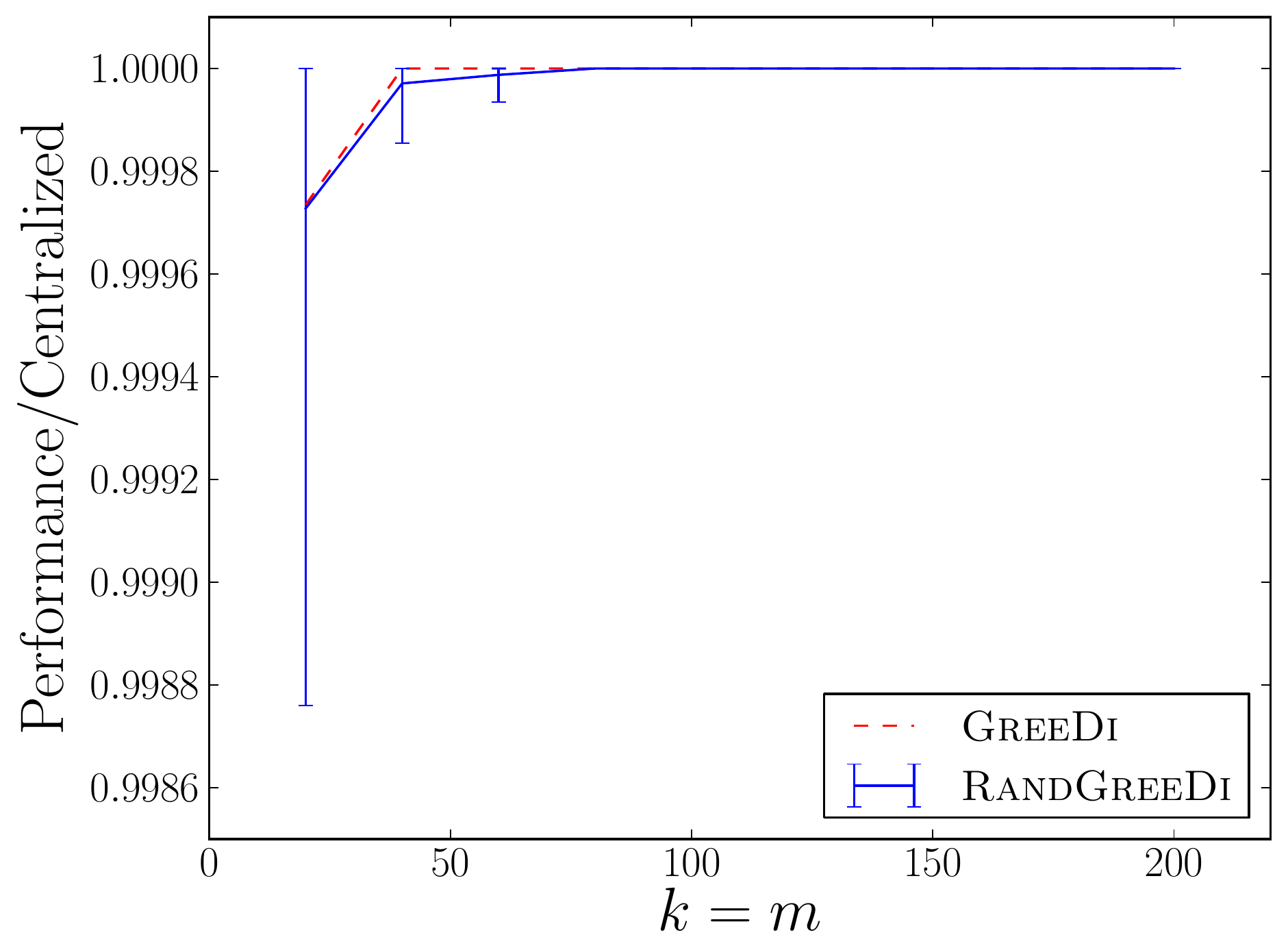}
\end{minipage}
\label{fig:ti-rand-det}
}
\subfigure[][Kosarak dataset]{
\begin{minipage}[t]{0.31\textwidth}
  \centering
  \includegraphics[scale=0.29]{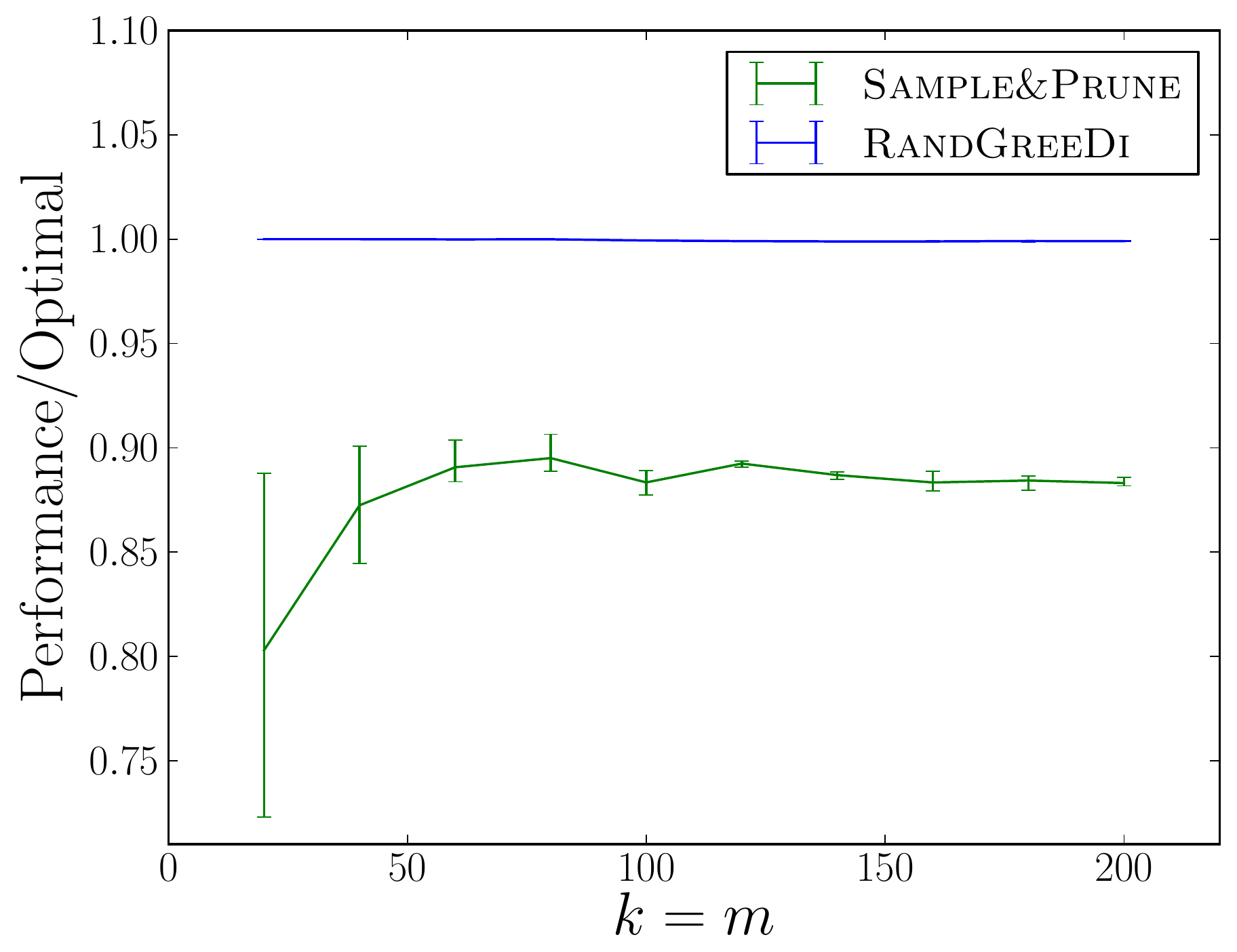}
\end{minipage}
\label{fig:kosarak-rand-sprune}
}
\subfigure[][accidents dataset]{
\begin{minipage}[t]{0.31\textwidth}
  \centering
  \includegraphics[scale=0.29]{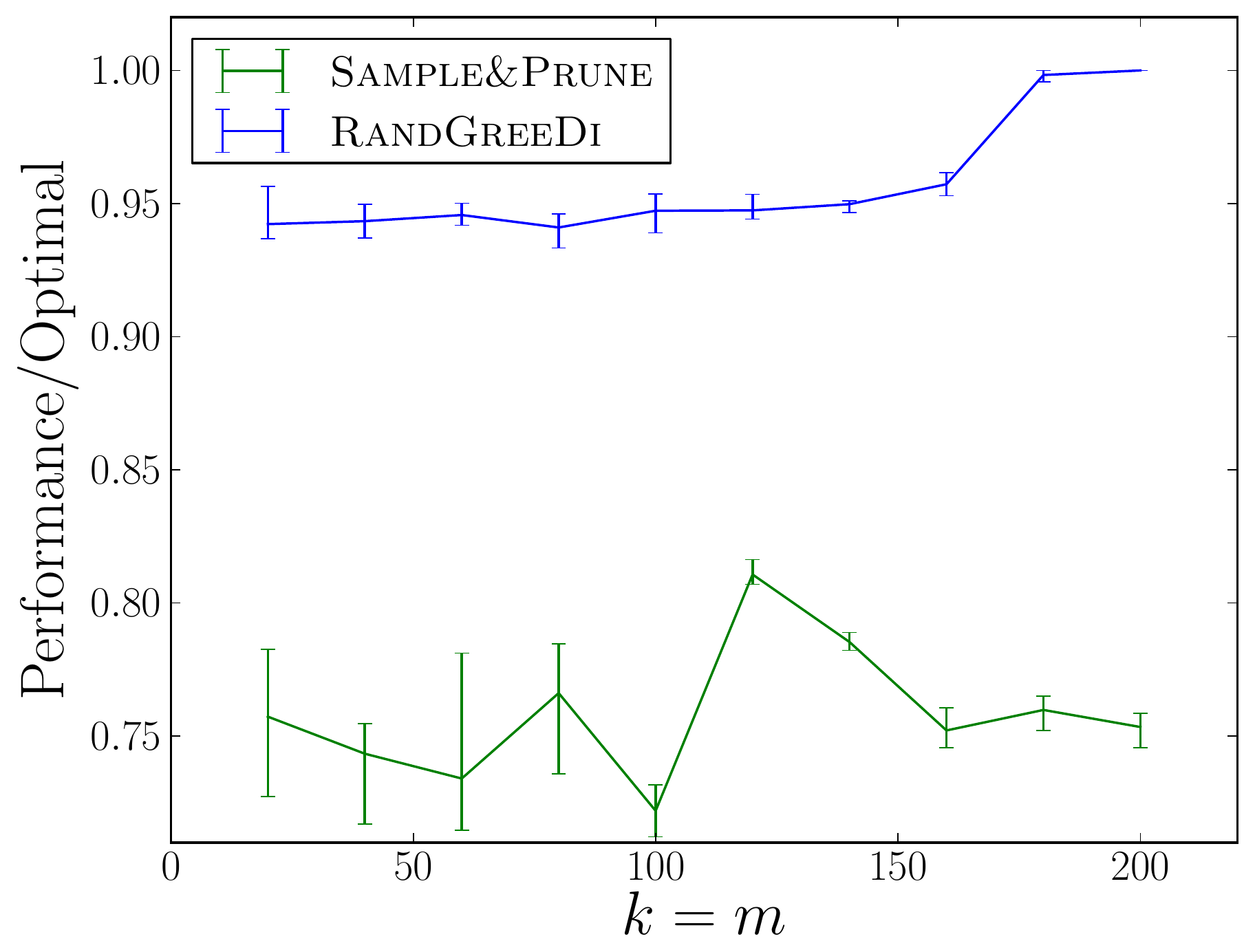}
\end{minipage}
\label{fig:acc-rand-sprune}
}
\subfigure[][10K tiny images]{
\begin{minipage}[t]{0.31\textwidth}
  \centering
  \includegraphics[scale=0.29]{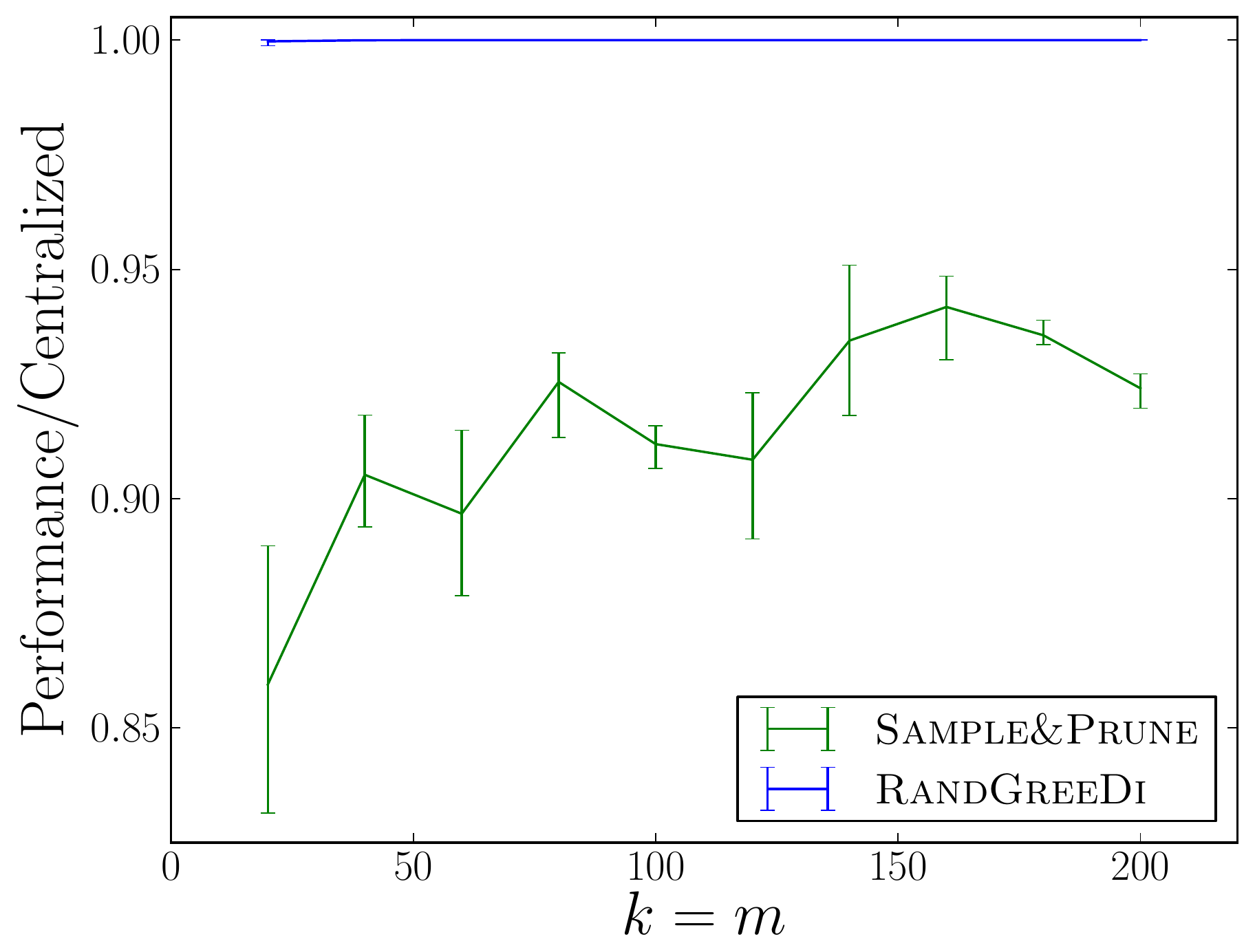}
\end{minipage}
\label{fig:ti-rand-sprune}
}
\subfigure[][synthetic diverse-yet-relevant instance ($n = 10000$,
$\lambda = n/k$)]{
\begin{minipage}[t]{0.31\textwidth}
  \centering
  \includegraphics[scale=0.29]{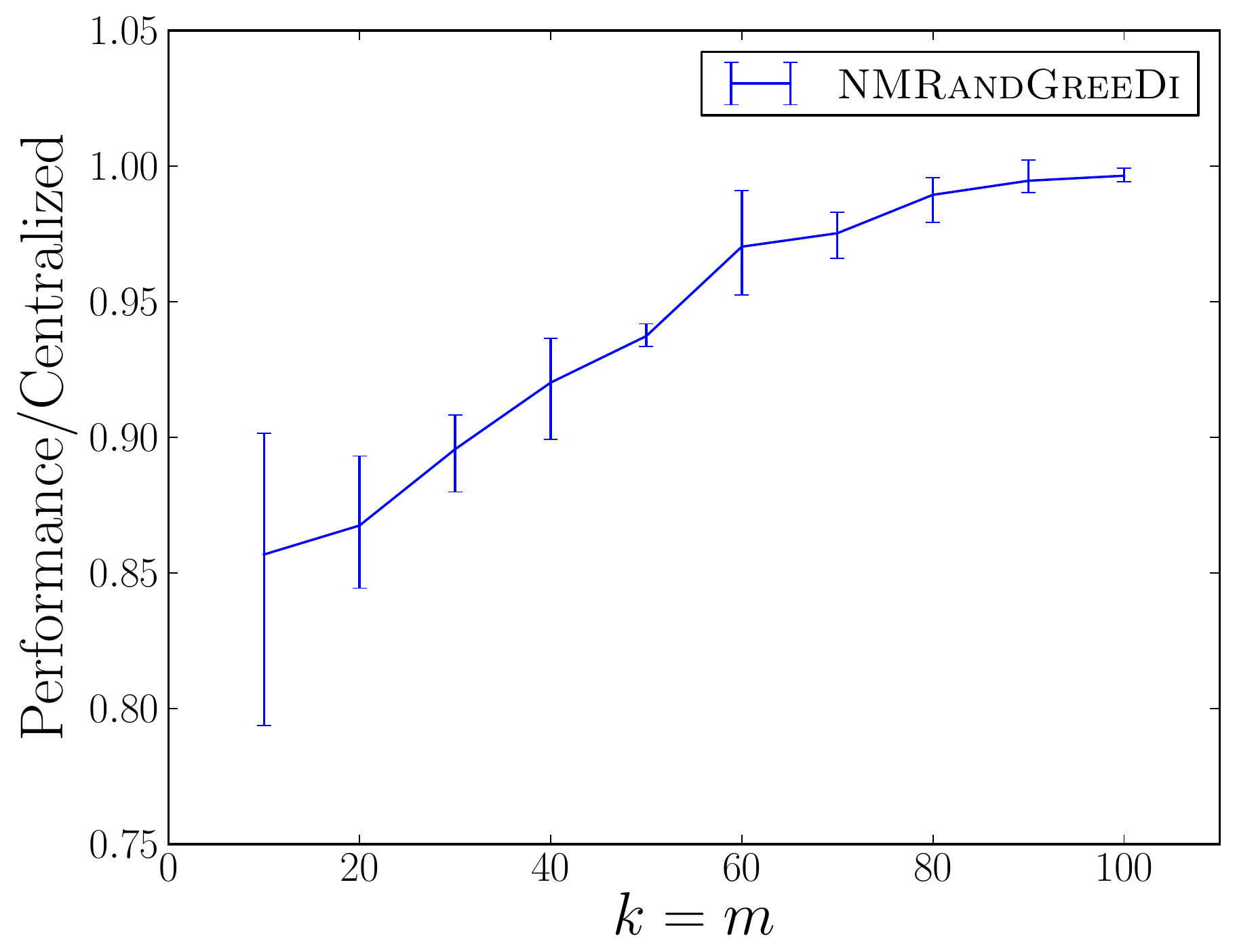}
\end{minipage}
\label{fig:dyr-synth-rand}
}
\subfigure[][synthetic hard instance for \GreeDI]{
\begin{minipage}[t]{0.31\textwidth}
  \centering
  \includegraphics[scale=0.29]{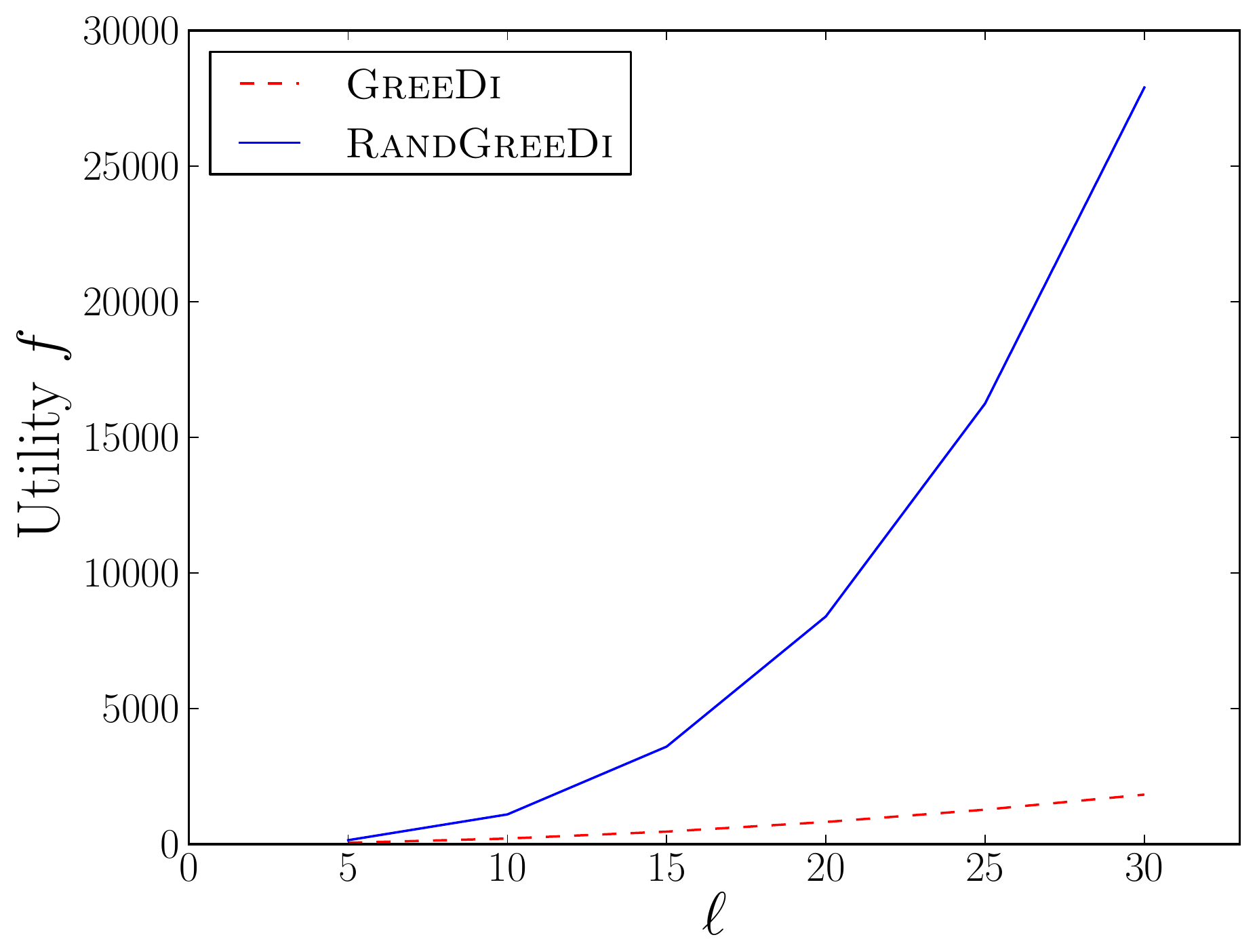}
\end{minipage}
\label{fig:tightInst}
}
\subfigure[][1M tiny images]{
\begin{minipage}[t]{0.31\textwidth}
  \centering
  \includegraphics[scale=0.29]{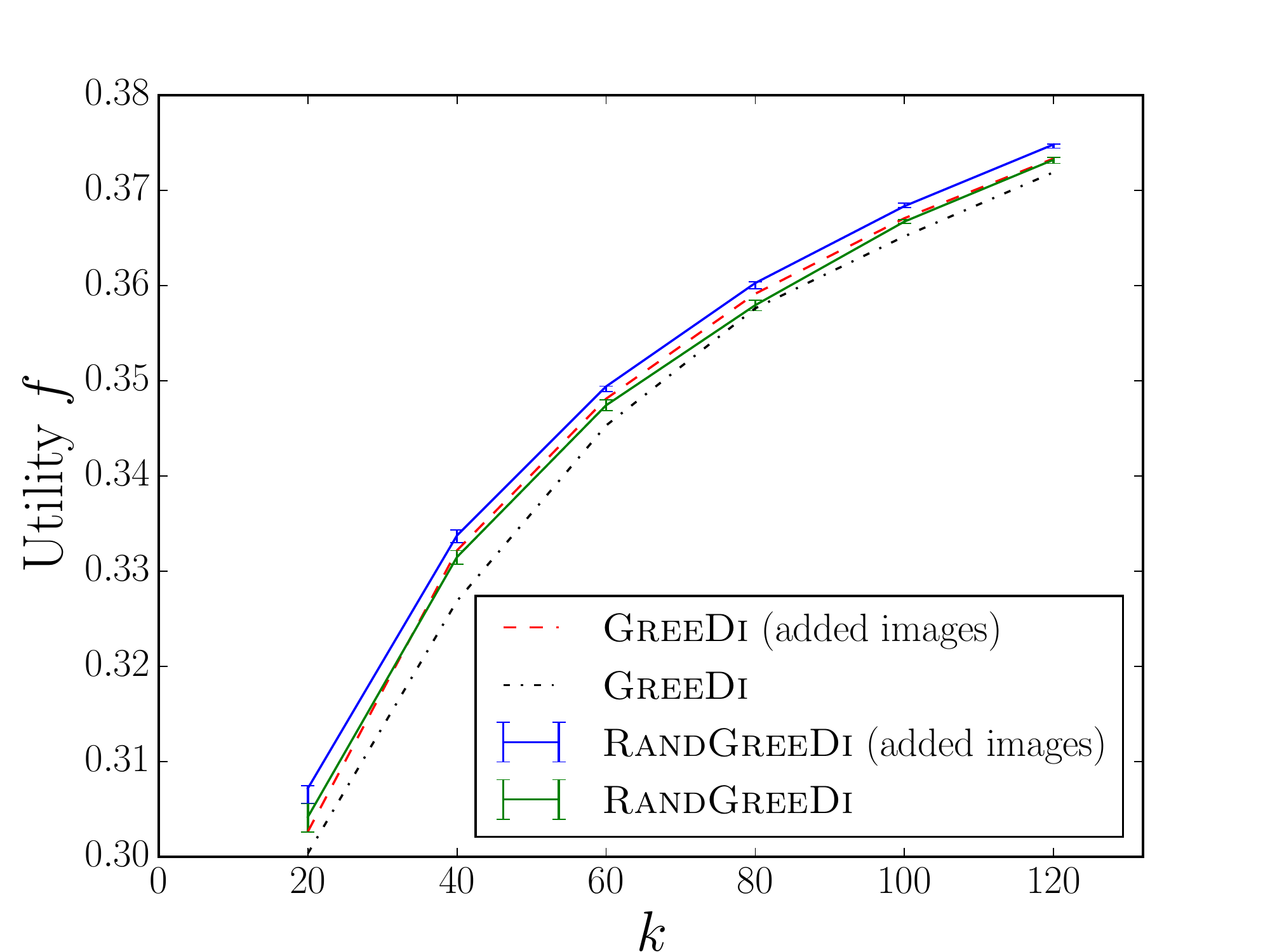}
\end{minipage}
\label{fig:million-images}
}

\subfigure[][matroid coverage $(n = 900, r=5)$]{
\begin{minipage}[t]{0.31\textwidth}
  \centering
  \includegraphics[scale=0.29]{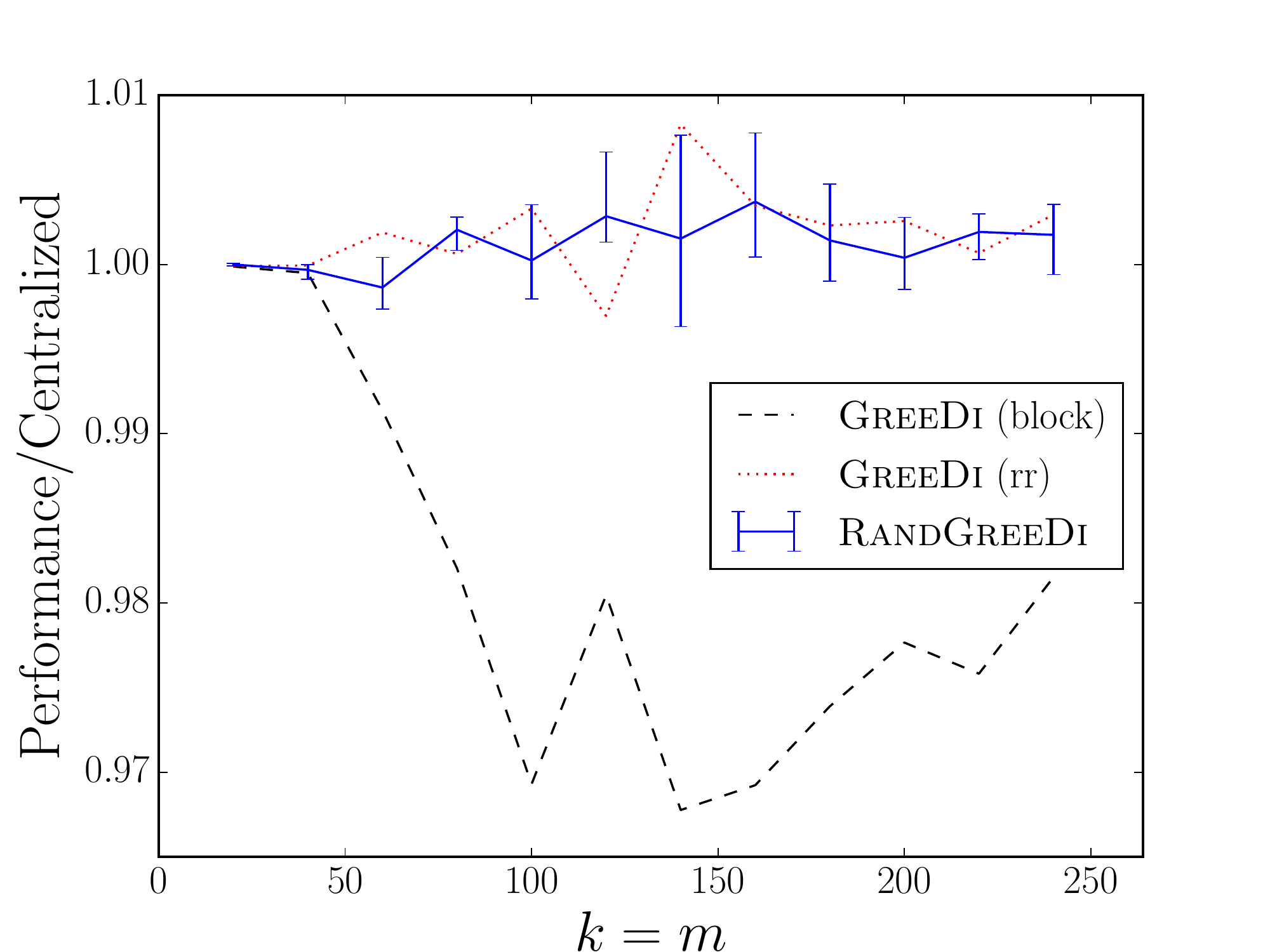}
\end{minipage}
\label{fig:matroid-coverage1}
}
\subfigure[][matroid coverage $(n= 100, r=100)$]{
\begin{minipage}[t]{0.33\textwidth}
  \centering
  \includegraphics[scale=0.29]{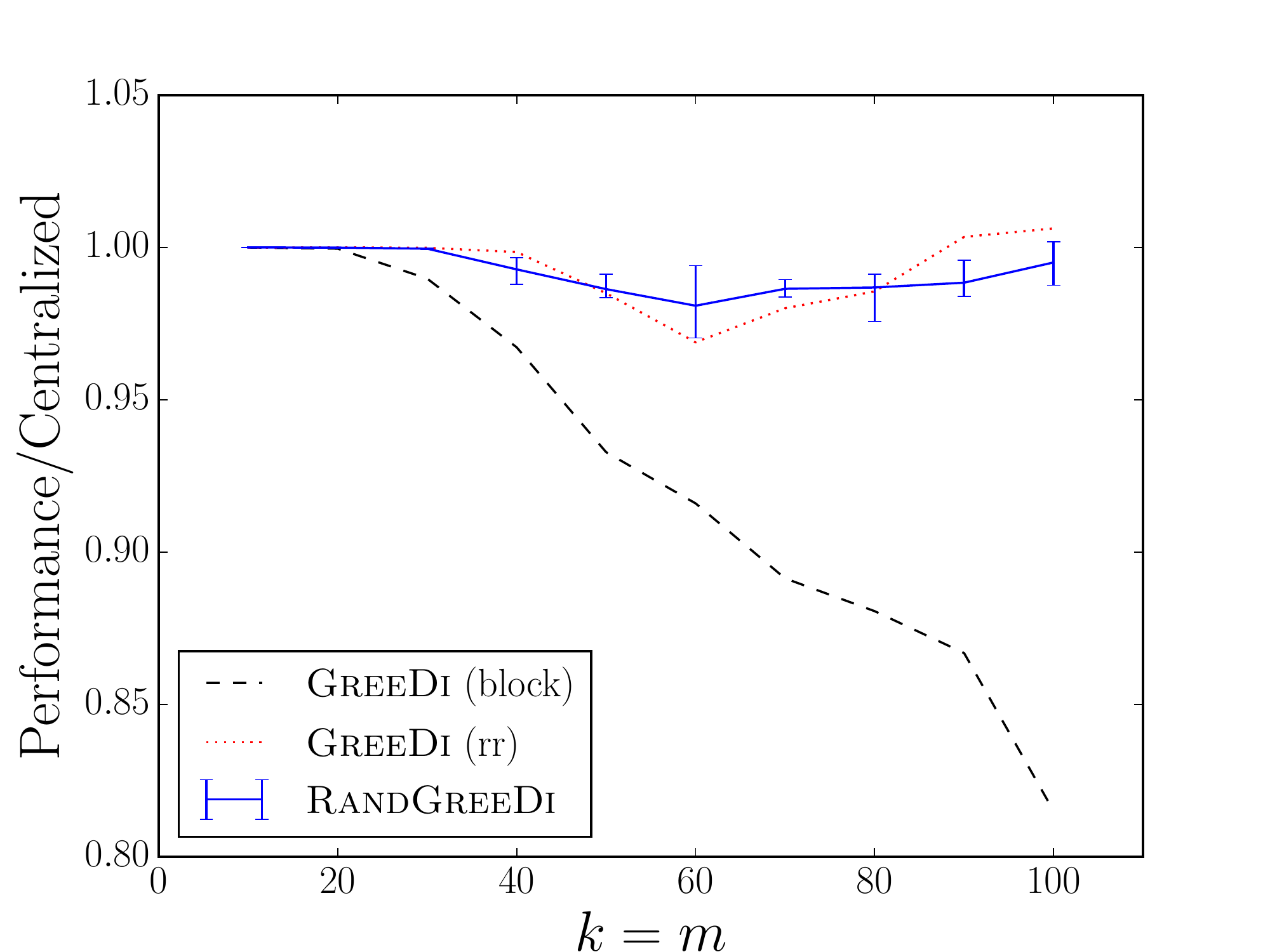}
\end{minipage}
\label{fig:matroid-coverage2}
}
\caption{Experimental Results}
\end{figure*}

We experimentally evaluate and compare the following distributed
algorithms for maximizing a monotone submodular function subject to a
cardinality constraint: the \RandGreeDI algorithm described in
Section~\ref{sec:rand-distr-greedy}, the deterministic \GreeDI
algorithm of \cite{MKSK13}, and the \SamplePrune algorithm of
\cite{KMVV13}. We run these algorithms in several scenarios and we
evaluate their performance relative to the centralized \Greedy
solution on the entire dataset.

{\bf Exemplar based clustering.}
Our experimental setup is similar to that of \cite{MKSK13}.  Our goal
is to find a representative set of objects from a dataset by solving
a $k$-medoid problem \cite{KaufmanR09} that aims to minimize the sum of
pairwise dissimilarities between the chosen objects and the entire
dataset. Let $V$ denote the set of objects in the dataset and let $d:
V \times V \rightarrow \R$ be a dissimilarity function; we assume
that $d$ is symmetric, that is, $d(i, j) = d(j, i)$ for each pair $i,
j$. Let $L: 2^V \rightarrow \R$ be the function such that $L(A) = {1
\over \card{V}} \sum_{v \in V} \min_{a \in A} d(a, v)$ for each set
$A \subseteq V$. We can turn the problem of minimizing $L$ into the
problem of maximizing a monotone submodular function $f$ by
introducing an auxiliary element $v_0$ and by defining $f(S) =
L(\set{v_0}) - L(S \cup \set{v_0})$ for each set $S \subseteq V$.

\emph{Tiny Images experiments:}
In our experiments, we used a subset of the Tiny Images dataset
consisting of $32 \times 32$ RGB images \cite{tinyimg}, each
represented as $3,072$ dimensional vector.  We subtracted from each
vector the mean value and normalized the result, to obtain a
collection of $3,072$-dimensional vectors of unit norm.  We
considered the distance function $d(x, y) = \|x - y\|^2$ for every
pair $x, y$ of vectors.  We used the zero vector as the auxiliary
element $v_0$ in the definition of $f$.

In our smaller experiments, we used 10,000 tiny images, and compared
the utility of each algorithm to that of the centralized greedy. The
results are summarized in Figures \ref{fig:ti-rand-det} and
\ref{fig:ti-rand-sprune}.

In our \emph{large scale experiments}, we used one million tiny
images, and $m = 100$ machines.  In the first round of the
distributed algorithm, each machine ran the \Greedy algorithm to
maximize a restricted objective function $f$, which is based on the
average dissimilarity $L$ taken over only those images assigned to
that machine.  Similarly, in the second round, the final machine
maximized an objective function $f$ based on the total dissimilarity
of all those images it received .  We also considered a variant
similar to that described by \cite{MKSK13}, in which 10,000
additional random images from the original dataset were added to the
final machine.  The results are summarized in Figure
\ref{fig:million-images}.

{\bf Remark on the function evaluation.} In decomposable cases such
as exemplar clustering, the function is a sum of distances over all
points in the dataset. By concentration results such as Chernoff
bounds, the sum can be approximated additively with high probability
by sampling a few points and using the (scaled) empirical sum. The
random subset each machine receives can readily serve as the samples
for the above approximation. Thus the random partition is useful for
for evaluating the function in a distributed fashion, in addition to
its algorithmic benefits.

{\bf Maximum Coverage experiments.}  We ran several experiments using
instances of the Maximum Coverage problem. In the Maximum Coverage
problem, we are given a collection $\sC \subseteq 2^V$ of subsets of
a ground set $V$ and an integer $k$, and the goal is to select $k$ of
the subsets in $\sC$ that cover as many elements as possible.

\emph{Kosarak and accidents datasets}\footnote{The data is available
at \url{http://fimi.ua.ac.be/data/}.}: We evaluated and compared the
algorithms on the datasets used by Kumar \etal \cite{KMVV13}. In both
cases, we computed the optimal centralized solution using CPLEX, and
calculated the actual performance ratio attained by the algorithms.
The results are summarized in Figures
\ref{fig:kosarak-rand-det-central}, \ref{fig:kosarak-rand-sprune},
\ref{fig:acc-rand-det-central}, \ref{fig:acc-rand-sprune}.

\emph{Synthetic hard instances:}
We generated a synthetic dataset with hard instances for the
deterministic \GreeDI. We describe the instances in
Section~\ref{app:det-greedi-tight}.  We ran the \GreeDI algorithm
with a worst-case partition of the data.  The results are summarized
in Figure \ref{fig:tightInst}.

{\bf Finding diverse yet relevant items.}
We evaluated our \NMRandGreeDI algorithm on the following instance of
\emph{non-monotone} submodular maximization subject to a cardinality
constraint. We used the objective function of Lin and Bilmes
\cite{LinB09}: $f(A) = \sum_{i \in V} \sum_{j \in A} s_{ij} -
\lambda \sum_{i, j \in A} s_{ij}$, where $\lambda$ is a redundancy
parameter and $\set{s_{ij}}_{ij}$ is a similarity matrix. We
generated an $n \times n$ similarity matrix with random entries
$s_{ij} \in \mathcal{U}(0, 100)$ and we set $\lambda = n/k$. The
results are summarized in Figure \ref{fig:dyr-synth-rand}.

{\bf Matroid constraints.}
In order to evaluate our algorithm on a matroid constraint, we
considered the following variant of maximum coverage: we are given a
space containing several demand points and $n$ facilities (e.g.
wireless access points or sensors).  Each facility can operate in one
of $r$ modes, each with a distinct coverage profile.  The goal is to
find a subset of at most $k$ facilities to activate, along with a
single mode for each activated facility, so that the total number of
demand points covered is maximized.  In our experiment, we placed
250,000 demand points in a grid in the unit square, together with a
grid of $n$ facilities.  We modeled coverage profiles as ellipses
centered at each facility with major axes of length $0.1\ell$, minor
axes of length $0.1/\ell$ rotated by $\rho$ where $\ell \in
\mathcal{N}(3,\frac{1}{3})$ and $\rho \in \mathcal{U}(0,2\pi)$ are
chosen randomly for each ellipse.  We performed two series of
experiments.  In the first, there were $n=900$ facilities, each with
$r=5$ coverage profiles, while in the second there were $n=100$
facilities, each with $r=100$ coverage profiles.

The resulting problem instances were represented as ground set
comprising a list of ellipses, each with a designated facility,
together with a partition matroid constraint ensuring that at most
one ellipse per facility was chosen. As in our large-scale
exemplar-based clustering experiments, we considered 3 approaches for
assigning ellipses to machines: assigning consecutive blocks of
ellipses to each machine, assigning ellipses to machines in
round-robin fashion, and assigning ellipses to machines uniformly at
random. The results are summarized in Figures
\ref{fig:matroid-coverage1} and \ref{fig:matroid-coverage2}; in these
plots, \GreeDI{(rr)} and \GreeDI{(block)} denote the results of
\GreeDI when we assign the ellipses to machines deterministically in
a round-robin fashion and in consecutive blocks, respectively.

In general, our experiments show that random and round robin are the
best allocation strategies.  One explanation for this phenomenon is
that both of these strategies ensure that each machine receives a few
elements from several distinct partitions in the first round.  This
allows each machine to return a solution containing several elements.

\medskip
{\large \bf Acknowledgements.} We thank Moran Feldman for suggesting
a modification to our original analysis that led to the simpler and
stronger analysis included in this version of the paper.

\bibliographystyle{plain}
\bibliography{paper}

\appendix

\section{Improved Deterministic GreeDI analysis}
\label{app:det-greedi-analysis}

\newcommand{\cB}{\mathcal{B}}
\newcommand{\cE}{\mathcal{E}}
\newcommand{\cM}{\mathcal{M}}
\newcommand{\cI}{\mathcal{I}}
\newcommand{\optb}{\tilde{\opt}}
\newcommand{\ground}{X}
\newcommand{\dummy}{D}
\newcommand{\ib}[2]{b_{#1}^{#2}}
\newcommand{\iB}[2]{B_{#1}^{#2}}
\newcommand{\tB}{\tilde{B}}
\newcommand{\tb}{\tilde{b}}
\newcommand{\tf}{\tilde{f}}
\newcommand{\tO}{\tilde{\opt}}

Let $\opt$ be an arbitrary collection of $k$ elements from $V$, and
let $M$ be the set of machines that have some element of $\opt$
placed on them.  For each $j \in M$ let $O_{j}$ be the set of
elements of $\opt$ placed on machine $j$, and let $r_j = |O_j|$ (note
that $\sum_{j \in M} r_j = k$).  Similarly, let $E_j$ be the set of
elements returned by the greedy algorithm on machine $j$.  Let $e_j^i
\in E_j$ denote the element chosen in the $i$th round of the greedy
algorithm on machine $j$, and let $E_j^i$ denote the set of all
elements chosen in rounds $1$ through $i$.  Finally, let $E = \cup_{j
\in M}E_j$, and $E^i = \cup_jE_j^i$.

We consider the marginal values:
\begin{align*}
x_j^i &= f_{E_j^{i-1}}(e_j^i) = f(E_j^i) - f(E_j^{i -1}) \\
y_j^i &= f_{E_j^{i-1}}(O_j),
\end{align*}
for each $1 \le i \le k$.  Note that because each element $e_j^i$ was
selected by in the $i$th round of the greedy algorithm on machine
$j$, we must have 
\begin{equation}
\label{eq:greedy}
x_j^i \ge \max_{o \in O_j}f_{E_j^{i-1}}(o) \ge \frac{y_j^i}{r_j}
\end{equation}
for all $j \in M$ and $i \in [k]$.  Moreover, the sequence
$x_j^1,\ldots,x_j^k$ is non-increasing for all $j \in M$.  Finally,
define $x_j^{k+1} = y_j^{k+1} = 0$ and $E_j^{k+1} = E_j^k$ for all
$j$.  We are now ready to prove our main claim.  
\begin{theorem}
\label{thm:greedi-improved}
Let $\tO \subseteq E$ be a set of $k$ elements from $E$ that
maximizes $f$.  Then,
\begin{equation*}
f(\opt) \le 2\sqrt{k}f(\tO).
\end{equation*}
\end{theorem}
\begin{proof}

For every $i \in [k]$ we have
\begin{align}
\label{eq:main}
f(\opt) &\le f(\opt \cup E^i) \notag \\
&= f(E^i) + f_{E^i}(\opt) \notag \\ 
&\le f(E^i) + \sum_{j \in M} f_{E^i}(O_j) \notag \\
&\le f(E^i) + \sum_{j \in M}f_{E_j^i}(O_j),
\end{align}
where the first inequality follows from monotonicity of $f$, and the
last two from submodularity of $f$.    

Let $i \le k$ be the smallest value such that:
\begin{equation}
\sum_{j \in M}r_j \cdot x_j^{i+1} \le \sqrt{k}\cdot \left[f(E^{i+1})
- f(E^{i})\right].
\label{eq:i-bound-simple}
\end{equation}
Note that some such value must $i$ must exist, since for $i = k$,
both sides are equal to zero.  We now derive a bound on each term on
the right of \eqref{eq:main}.

\begin{lemma}
\label{lem:1}
$\sum_{j \in M} f(E_j^i) \le \sqrt{k}\cdot f(\tO)$.
\end{lemma}
\begin{proof}
Because $i$ is the smallest value for which \eqref{eq:i-bound-simple}
holds, we must have
\[
\sum_{j \in M}r_j \cdot x_j^\ell > \sqrt{k}\cdot \left[f(E^\ell) -
f(E^{\ell-1})\right], \mbox{ for all $\ell \le i$.}
\]
Therefore,
\begin{align*}
\sum_{j \in M} r_j \cdot f(E_j^i) &= \sum_{j \in M} \sum_{\ell =
1}^ir_j \cdot \left[f(E_j^\ell) - f(E_j^{\ell - 1})\right] \\
&= \sum_{j \in M} \sum_{\ell = 1}^i r_j \cdot x_j^i  \\
&= \sum_{\ell = 1}^i \sum_{j \in M} r_j \cdot x_j^i \\
&> \sum_{\ell = 1}^i\sqrt{k} \cdot \left[f(E^{\ell}) - f(E^{\ell -
1})\right] \\
&= \sqrt{k} \cdot f(E^i),
\end{align*}
and so, 
\begin{equation*}
f(E^i) < \frac{1}{\sqrt{k}}\sum_{j \in M}r_j \cdot f(E^i_j) 
\le \frac{1}{\sqrt{k}}\sum_{j \in M}r_j \cdot f(E_j)
\le \frac{1}{\sqrt{k}}\sum_{j \in M}r_j \cdot f(\tO)
= \sqrt{k}\cdot f(\tO). \qedhere
\end{equation*}
\end{proof}

\begin{lemma}
\label{lem:2}
$\sum_{j \in M}f_{E_j^i(O_j)} \le \sqrt{k} \cdot f(\tO)$.
\end{lemma}
\begin{proof}
We consider two cases:
\paragraph{Case: $i < k$.} We have $i + 1 \le k$, and by
\eqref{eq:greedy} we have $f_{E_j^i}(O_j) = y_j^{i + 1} \le r_j\cdot
x_j^{i+1}$ for every machine $j$.  Therefore:

\begin{align*}
\sum_{j \in M}f_{E_j^i}(O_j)
&\le 
\sum_{j \in M}r_j \cdot x_j^{i+1} \\
&\le 
\sqrt{k}\cdot (f(E^{i+1}) - f(E^{i})) \\
&=
\sqrt{k}\cdot f_E^{i}(E^{i+1} \setminus E^i) \\
&\le 
\sqrt{k}\cdot f(E^{i+1} \setminus E^i) \\
&\le
\sqrt{k}\cdot f(\tO).
\end{align*}

\paragraph{Case: $i = k$.}  By submodularity of $f$ and
\eqref{eq:greedy}, we have
\[f_{E_j^i}(O_j) \le f_{E_j^{k-1}}(O_j) = y_j^{k} \le r_j \cdot x_j^{k}.\]
Moreover, since the sequence $x_j^1,\ldots,x_j^k$ is nonincreasing
for all $j$, 
\[
x_j^k \le \frac{1}{k}\sum_{i = 1}^k x_j^i = \frac{1}{k}\cdot f(E_j).
\]
Therefore,
\begin{equation*}
\sum_{j \in M}f_{E_j^i}(O_j)
\le
\sum_{j \in M}\frac{r_j}{k}\cdot f(E_j)
\le
\sum_{j \in M}\frac{r_j}{k}\cdot f(\tO)
= f(\tO).
\end{equation*}
Thus, in both cases, we have $\sum_{j \in M}f_{E_j^i}(O_j) \le
\sqrt{k} \cdot f(\tO)$ as required.
\end{proof}
Applying Lemmas \ref{lem:1} and \ref{lem:2} to the right of
\eqref{eq:main}, we obtain
\[
f(\opt) \le 2 \sqrt{k} \cdot f(\tO),
\]
completing the proof of Theorem \ref{thm:greedi-improved}.
\end{proof}

\begin{corollary}
The distributed greedy algorithm gives a $\frac{(1-1/e)}{2\sqrt{k}}$
approximation for maximizing a monotone submodular function subject
to a cardinality constraint $k$, regardless of how the elements are
distributed.
\end{corollary}

\section{A tight example for Deterministic GreeDI}
\label{app:det-greedi-tight}

Here we give a family of examples that show that the GreeDI algorithm
of Mirzasoleiman \etal cannot achieve an approximation better than
$1/\sqrt{k}$.

Consider the following instance of Max $k$-Coverage. We have
${\ell}^2 + 1$ machines and $k = \ell + {\ell}^2$. Let $N$ be a
ground set with ${\ell}^2 + {\ell}^3$ elements, $N = \set{1, 2,
\dots, {\ell}^2 + {\ell}^3}$. We define a coverage function on a
collection $\mathcal{S}$ of subsets of $N$ as follows. In the
following, we define how the sets of $\mathcal{S}$ are partitioned on
the machines.

On machine $1$, we have the following $\ell$ sets from $\opt$: $O_1 =
\set{1, 2, \dots, \ell}$, $O_2 = \set{\ell + 1, \dots, 2\ell}$,
\dots, $O_{\ell} = \set{{\ell}^2 - \ell + 1, \dots, {\ell}^2}$. We
also pad the machine with copies of the empty set.

On machine $i > 1$, we have the following sets. There is a single set
from $\opt$, namely $O'_i = \set{{\ell}^2 + (i - 1) \ell + 1,
{\ell}^2 + (i - 1) \ell + 2, \dots, {\ell}^2 + i \ell}$.
Additionally, we have $\ell$ sets that are designed to fool the
greedy algorithm; the $j$-th such set is $O_j \cup \set{ {\ell}^2 +
(i - 1) \ell + j}$. As before, we pad the machine with copies of the
empty set.

The optimal solution is $O_1, \dots, O_{\ell}$, $O'_1, \dots, O'_{
{\ell}^2}$ and it has a total coverage of ${\ell}^2 + {\ell}^3$.

On the first machine, Greedy picks the $\ell$ sets $O_1, \dots, O_m$
from $\opt$ and ${\ell}^2$ copies of the empty set. On each machine
$i > 1$, Greedy first picks the $\ell$ sets $A_j = O_j \cup \set{
{\ell}^2 + (i - 1) \ell + j}$, since each of them has marginal value
greater than $O'_i$. Once Greedy has picked all of the $A_j$'s, the
marginal value of $O'_i$ becomes zero and we may assume that Greedy
always picks the empty sets instead of $O'_i$.

Now consider the final round of the algorithm where we run Greedy on
the union of the solutions from each of the machines. In this round,
regardless of the algorithm, the sets picked can only cover $\set{1,
\ldots, {\ell}^2}$ (using the set $O_1, \dots, O_{\ell}$) and one
additional item per set for a total of $2 {\ell}^2$ elements. Thus
the total coverage of the final solution is at most $2 {\ell}^2$.
Hence the approximation is at most ${2 \ell^2 \over \ell^2 + \ell^3}
= {2 \over 1 + \ell} \approx {1 \over \sqrt{k}}$.
\end{document}